\documentclass{article}

\usepackage{PRIMEarxiv}

\usepackage[utf8]{inputenc}
\usepackage{amsmath,amssymb,amsthm,mathtools}
\usepackage{booktabs}
\usepackage{graphicx}
\pdfmapfile{=pdftex35.map}
\pdfmapfile{+cm.map}
\pdfmapfile{+cmextra.map}
\pdfmapfile{+symbols.map}
\usepackage[protrusion=true,expansion=false]{microtype}
\usepackage{longtable}
\usepackage{multirow}
\usepackage{array}
\usepackage[table]{xcolor}
\usepackage{url}
\usepackage[section]{placeins}
\usepackage{enumitem}
\usepackage[numbers,sort&compress]{natbib}
\usepackage{hyperref}
\hypersetup{hidelinks,pdftitle={Drift Variation autoencoder: Unifying Generation and Representation Learning through Conditional Posterior Flow Matching},pdfauthor={Jiarui Cao}}
\graphicspath{{figures/}}

\newcommand{\R}{\mathbb R}
\newcommand{\E}{\mathbb E}
\newcommand{\KL}{\operatorname{KL}}

\newcommand{\cK}{\mathcal K}
\newcommand{\cL}{\mathcal L}
\newcommand{\cR}{\mathcal R}
\newcommand{\cT}{\mathcal T}

\newcommand{\bc}{\mathbf c}
\newcommand{\bz}{\mathbf z}
\newcommand{\beps}{\boldsymbol\varepsilon}
\newcommand{\bX}{\mathbf X}
\newcommand{\bY}{\mathbf Y}
\newcommand{\bC}{\mathbf C}
\newcommand{\bZ}{\mathbf Z}
\newcommand{\bA}{\mathbf A}
\newcommand{\bS}{\mathbf S}
\newcommand{\dd}{\mathrm d}
\newcommand{\norm}[1]{\left\lVert #1\right\rVert}
\newcommand{\indep}{\perp\!\!\!\perp}

\newcolumntype{L}[1]{>{\raggedright\arraybackslash}p{#1}}

\newtheorem{theorem}{Theorem}[section]
\newtheorem{lemma}[theorem]{Lemma}
\newtheorem{proposition}[theorem]{Proposition}
\newtheorem{corollary}[theorem]{Corollary}

\newtheorem{definition}[theorem]{Definition}

\title{Drift Variation Autoencoder:\\Unifying Generation and Representation Learning through Conditional Posterior Flow Matching}

\author{Jiarui Cao\\
The Chinese University of Hong Kong\\
\texttt{1155244613@link.cuhk.edu.hk}}
\date{}

\begin{document}
\maketitle

\begin{abstract}
Stochastic masking, cropping, or modality removal makes deterministic reconstruction an incomplete target: one observation can admit many clean completions. This work takes the corresponding posterior $P(X\mid C)$ as the common statistical object for conditional generation and generatively sufficient representation learning. Drift Variation autoencoder trains a masked encoder $Z=E(C)$ and a conditional flow decoder with one clean-prediction Flow Matching loss. The analysis first decomposes the ideal conditional KL into generator approximation and the representation deficiency $I(X;C\mid Z)$. It then derives orthogonal risk decompositions for conditional Flow Matching. For an affine Gaussian path, the clean-prediction representation gap is zero if and only if $P(X\mid Z)=P(X\mid C)$. Thus the encoder-dependent excess clean-prediction risk induced by Flow Matching and the profiled ideal conditional KL have the same posterior-sufficient zero set, without being numerically equal objectives. An exact conditional field with a zero-noise endpoint then generates $P(X\mid Z)$ and hence $P(X\mid C)$ at a joint ideal optimum. The result extends to continuous multimodal product spaces when the complete modality tuple remains the Flow target for every observation mask. On CrossGeom-4, an 18-run controlled benchmark, observable factors have linear-probe $R^2$ of $0.9990$-$0.9992$, shuffling the joint model's encoder condition increases conditional error by $13.5\times$-$15.7\times$, and joint target attention reduces disagreement on an unobserved factor shared by two outputs by $90.1$-$92.8\%$ relative to independent target decoders. Visible modalities are also generated and reconstructed, directly validating the full-tuple objective. Unconditional mode balance remains imperfect, delimiting the empirical claim to a controlled multimodal proof of concept.
\end{abstract}

\section{Introduction}

When an augmentation removes information, reconstruction is not a one-to-one problem. A crop can hide several plausible backgrounds; a masked signal can admit several continuations; and an observed modality can leave residual variables that two missing modalities must share. The statistically correct target is therefore not a single reconstruction but the posterior $P(X\mid C)$ of a clean sample $X$ given a stochastic observation $C=A(X,\Xi_A)$. This posterior offers a precise common target for two tasks that are usually trained separately: a generator should sample it, while an encoder representation should preserve exactly the information in $C$ needed to specify it.

Current pipelines typically optimize the two capabilities with distinct objectives. Self-supervised encoders use view agreement, masked reconstruction, or latent prediction~\cite{chen2020simple,he2022masked,assran2023self,oquab2023dinov2}, whereas diffusion and flow models estimate denoising scores or transport fields~\cite{ho2020denoising,song2020score,lipman2022flow}. A downstream system may then align, freeze, distill, or reconnect the components. This separation can be useful, but it does not make representation correctness a consequence of the same uncertainty modeled by the generator. Conversely, a conditional generator can reduce denoising loss through its noisy state without exposing a reusable representation. This work asks whether one ordinary generative objective can instead \emph{characterize and train} the representation it consumes.

Let $Z=E_\theta(C)$ be a deterministic representation and let $Q(X\mid Z)$ be a conditional generator. The ideal posterior-matching risk is
\begin{equation}
\cK(E,Q)=\E_C\KL\!\left(P(X\mid C)\,\|\,Q(X\mid E(C))\right).
\label{eq:intro-kl}
\end{equation}
Because $Z$ is a function of $C$, this risk decomposes exactly as
\begin{equation}
\cK(E,Q)=I(X;C\mid Z)+\E_Z\KL\!\left(P(X\mid Z)\,\|\,Q(X\mid Z)\right).
\label{eq:intro-kl-decomp}
\end{equation}
The generator must model the posterior available after compression, while the representation is correct exactly when $P(X\mid Z)=P(X\mid C)$. This condition is called \emph{posterior sufficiency}. It is a generative notion of representation quality, not a guarantee of semantic minimality, disentanglement, or invariance. An identity encoder can be sufficient; the observation process and bottleneck determine which sufficient representations are useful. Equation~\eqref{eq:intro-kl} defines the desired statistical object, but its unknown conditional densities prevent direct optimization.

Drift Variation autoencoder turns this posterior principle into a sample-based objective using Conditional Flow Matching (CFM)~\cite{lipman2022flow,lipman2024flow}. The training procedure draws a condition-independent Gaussian source $X_0$, forms an affine path $X_t=\alpha_tX+\sigma_tX_0$, and jointly trains the encoder and a clean-prediction decoder that receives only $(X_t,t,Z)$. The sample target is analytic and fixed by the observed data pair and probability path; no teacher, contrastive target, or generated reference set is required by the population objective.

The key difficulty is showing that CFM constrains the encoder rather than only the decoder. The analysis derives an orthogonal risk decomposition into irreducible path variance, representation deficiency, and model approximation. For clean prediction, the encoder-dependent term is
\begin{equation}
\Delta_{\rm rep}^{x}(E)=\E\!\left[w(t)\norm{\E[X\mid X_t,t,C]-\E[X\mid X_t,t,Z]}^2\right].
\label{eq:intro-rep-gap}
\end{equation}
For an affine Gaussian path with positive effective mass at interior noise levels, the analysis proves
\begin{equation}
\Delta_{\rm rep}^{x}(E)=0
\quad\Longleftrightarrow\quad
P(X\mid Z)=P(X\mid C)
\quad\Longleftrightarrow\quad
I(X;C\mid Z)=0.
\label{eq:intro-main}
\end{equation}
The proof recovers a noisy conditional score from the Gaussian posterior mean, then the smoothed density, and finally the clean posterior through injectivity of Gaussian convolution. This is an equivalence of ideal encoder zero sets, not an equality between the numerical CFM and KL objectives or their encoder gradients. With an exact representation-conditioned field, the conditional ODE generates $P(X\mid Z)$; Eq.~\eqref{eq:intro-main} closes the loop to $P(X\mid C)$.

The formulation extends to a multimodal tuple $\bX=(X^{(1)},\ldots,X^{(M)})$. An observation mask controls which clean modalities enter the encoder, but the Flow target remains the complete tuple for every mask. Each modality receives source noise, contributes a positive loss weight, and evolves in the joint decoder. The product-space theorem then identifies the complete joint posterior, including residual dependence between simultaneously generated modalities. The evaluation tests this prediction on CrossGeom-4, a controlled three-modality distribution with known shared factors and Euclidean, spherical, and hyperbolic views. Joint and independent decoders have similar deterministic conditional accuracy, but only joint target attention coordinates the same unobserved posterior draw across two outputs.

The closest theoretical comparison is Zero-Flow Encoders~\cite{wang2026zero}, which constructs a separate sufficient-representation loss from a midpoint zero-flow criterion under an independent coupling. Drift Variation autoencoder instead studies the ordinary conditional decoder used for generation: its clean-prediction risk itself induces the encoder criterion, and Gaussian denoiser identifiability links that criterion to posterior sufficiency.

The main contributions are:
\begin{itemize}[leftmargin=2em,itemsep=3pt]
\item \textbf{Posterior principle.} Generative representation correctness is defined by $P(X\mid Z)=P(X\mid C)$, and the ideal conditional KL is decomposed into representation deficiency and generator approximation.
\item \textbf{Flow risk and sufficiency.} Exact velocity and clean-prediction risk decompositions are derived, and it is proved that the encoder-dependent clean-prediction representation gap induced by affine-Gaussian CFM and the profiled conditional KL share the same posterior-sufficient zero set.
\item \textbf{Endpoint and multimodal guarantees.} An exact field is connected to posterior generation, and the result is extended to complete continuous multimodal tuples under any observed subset.
\item \textbf{Controlled validation.} CrossGeom-4 measures representation use, visible-stream reconstruction, conditional completion, residual joint coupling, and unconditional coverage across three geometries and three seeds.
\end{itemize}

\section{Related Work}

\paragraph{Self-supervised representation learning.} Contrastive and self-distillation methods align augmented views, while masked and predictive methods recover pixels, tokens, or target embeddings \cite{chen2020simple,caron2021emerging,he2022masked,assran2023self,oquab2023dinov2}. MAE shows that high-ratio masking can learn scalable visual features, whereas I-JEPA deliberately predicts in representation space to emphasize semantic structure~\cite{he2022masked,assran2023self}. These methods choose an invariance or prediction target. Drift Variation autoencoder instead defines a generatively sufficient representation by the clean-data posterior induced by an observation. It does not require all views of one sample to have identical features and does not claim minimal or semantic sufficiency.

\paragraph{Generation and representation.} Denoising autoencoders connect reconstruction fields to data scores \cite{vincent2011connection,alain2014regularized}. Diffusion Autoencoders \cite{preechakul2022diffusion}, representation learning with diffusion \cite{traub2022representation,mittal2023diffusion}, and DDAE \cite{xiang2023denoising} study representations learned with or extracted from generative models; later analysis tracks how diffusion representations change with noise and training~\cite{li2026understanding}. In the other direction, REPA aligns generative hidden states with a pretrained encoder \cite{yu2024representation}. Representation Autoencoders and RepTok import or adapt self-supervised latent spaces for efficient generation \cite{zheng2026diffusion,gui2026adapting}. These approaches demonstrate the value of semantic geometry but retain an external representation source, an auxiliary alignment objective, or a separate latent-space construction. Drift Variation autoencoder asks when the conditional generative loss itself identifies the required representation.

\paragraph{Flow Matching and conditional sufficiency.} Flow Matching regresses vector fields of fixed probability paths, and Conditional Flow Matching replaces an inaccessible marginal field by tractable conditional targets with the same population minimizer for the field model \cite{lipman2022flow,lipman2024flow}. Rectified Flow studies straight couplings and efficient transport~\cite{liu2022flow}. FlowFM jointly trains an encoder and conditional flow on wearable signals \cite{ukita2025high}; Self-Flow adds dual-timestep information asymmetry and a self-supervised feature objective for scalable multimodal synthesis \cite{chefer2026self}; and Symmetrical Flow Matching combines opposing flows with supervised tasks~\cite{caetano2026symmetrical}. Closest to the theory developed here, Zero-Flow Encoders relate a midpoint zero-flow criterion to conditional independence and construct a separate representation loss \cite{wang2026zero}. Drift Variation autoencoder instead proves that, after profiling the decoder, the encoder-dependent excess clean-prediction risk of the same conditional decoder used for generation has posterior-sufficient encoders as its ideal zero set. Its proof uses Gaussian denoiser identifiability rather than a midpoint zero-flow criterion. Finite-error bounds from flow-field error to terminal KL require additional regularity \cite{su2025flow}; the present result concerns exact zero sets and endpoints.

\paragraph{Multimodal learning.} MultiMAE and 4M randomize masked inputs and targets across modalities, with 4M supporting flexible tokenized conditional generation \cite{bachmann2022multimae,mizrahi20234m}. Chameleon, Transfusion, and Janus combine multimodal understanding and generation through early fusion, autoregressive-diffusion objectives, or decoupled visual encoders \cite{team2024chameleon,zhou2025transfusion,wu2025janus}. Drift Variation autoencoder contributes a distribution-level abstraction for continuous latents: any observed subset defines the condition, while all modality streams remain one joint Flow target. This distinction makes cross-output residual dependence, rather than only per-modality fidelity, an explicit empirical requirement.

\section{Posterior Flow Matching for Representation Learning}
\label{sec:method}

\subsection{The augmentation posterior is the common target}

Let $X\in\R^d$ follow the data distribution, let $\Xi_A$ denote augmentation randomness, and define the observed context and deterministic representation by
\begin{equation}
C=A(X,\Xi_A),\qquad Z=E_\theta(C).
\label{eq:setup}
\end{equation}
The augmentation may mask patches, crop spatial support, alter photometric statistics, remove a modality, or combine several such operations. If it is not invertible, $P(X\mid C=c)$ need not be a point mass. The modeling target is therefore the complete conditional law rather than one reconstruction.

\begin{definition}[Posterior-sufficient representation]
\label{def:sufficient}
An encoder $E$ is sufficient for the clean target $X$ relative to the observation $C$ when, for $Z=E(C)$,
\begin{equation}
P(X\mid Z)=P(X\mid C)\qquad\text{almost surely}.
\label{eq:sufficiency}
\end{equation}
Equivalently, $X\indep C\mid Z$ or $I(X;C\mid Z)=0$ whenever the conditional mutual information is well defined.
\end{definition}

Sufficiency is deliberately weaker than minimal sufficiency. The identity encoder $Z=C$ can be sufficient, and nothing in Definition~\ref{def:sufficient} alone forces semantic compression, disentanglement, or invariance to every augmentation. Those properties must come from the observation process, an architectural bottleneck, regularization, or an additional criterion.

Consider the ideal conditional distribution-matching objective
\begin{equation}
\cK(E,Q)=\E_C\KL\!\left(P(X\mid C)\,\|\,Q(X\mid Z)\right),\qquad Z=E(C).
\label{eq:ideal-kl}
\end{equation}

\begin{proposition}[Conditional KL decomposition]
\label{prop:kl-decomp}
If the conditional laws and displayed KL divergences exist, then
\begin{equation}
\boxed{\cK(E,Q)=I(X;C\mid Z)+\E_Z\KL\!\left(P(X\mid Z)\,\|\,Q(X\mid Z)\right).}
\label{eq:kl-decomp}
\end{equation}
Consequently, for an unrestricted conditional generator family,
\begin{equation}
\inf_Q\cK(E,Q)=I(X;C\mid Z),
\label{eq:kl-profile}
\end{equation}
and the infimum is zero exactly for posterior-sufficient encoders.
\end{proposition}

Proposition~\ref{prop:kl-decomp} makes the desired synchronization explicit. The generator must model the law available after compression, $P(X\mid Z)$, while the encoder is correct only if this law retains the posterior specified by the original observation. Direct evaluation of Eq.~\eqref{eq:ideal-kl} is unavailable from ordinary joint samples because $P(X\mid C)$ has no tractable density. The next step constructs a probability path whose regression objective is sample based.

\subsection{A conditional probability path from common noise}

Draw a source independent of both data and augmentation,
\begin{equation}
X_0=s_0\varepsilon,\qquad \varepsilon\sim\mathcal N(0,I_d),\qquad X_0\indep(X,C),\qquad s_0>0.
\label{eq:source}
\end{equation}
For differentiable scalar schedules $\alpha_t$ and $\sigma_t$, define
\begin{equation}
X_t=\alpha_tX+\sigma_tX_0,
\qquad
U_t=\dot\alpha_tX+\dot\sigma_tX_0.
\label{eq:path-general}
\end{equation}
The principal affine path uses
\begin{equation}
\alpha_t=t,\qquad \sigma_t=1-(1-\sigma_{\min})t,
\label{eq:affine-schedules}
\end{equation}
so that
\begin{equation}
X_t=tX+\sigma_tX_0,\qquad U_t=X-(1-\sigma_{\min})X_0=\frac{X-(1-\sigma_{\min})X_t}{\sigma_t}.
\label{eq:affine-velocity}
\end{equation}
At $t=0$, every condition shares $P_0=\mathcal N(0,s_0^2I)$. If $\sigma_{\min}=0$, then $X_1=X$ and the endpoint conditioned on $C$ is exactly $P(X\mid C)$. If $\sigma_{\min}>0$, the endpoint is the Gaussian-smoothed posterior $P(X\mid C)*\mathcal N(0,\sigma_{\min}^2s_0^2I)$.

For a fixed clean endpoint $x$, Eq.~\eqref{eq:path-general} defines the tractable bridge
\begin{equation}
p_t(x_t\mid x)=\mathcal N\!\left(x_t;\alpha_tx,\sigma_t^2s_0^2I\right).
\label{eq:bridge}
\end{equation}
Marginalizing the endpoint over the augmentation posterior gives the actual condition-level path
\begin{equation}
p_t(x_t\mid c)=\int p_t(x_t\mid x)p(x\mid c)\,\dd x.
\label{eq:conditional-path}
\end{equation}
Thus one training pair $(c,x)$ is a Monte Carlo observation from the endpoint mixture, not a declaration that $P(X\mid C=c)$ is the Dirac mass at $x$.

Define the marginal conditional velocity
\begin{equation}
u_C(x_t,t,c)=\E[U_t\mid X_t=x_t,t,C=c].
\label{eq:u-c}
\end{equation}

\begin{proposition}[Conditional path field]
\label{prop:conditional-path}
Assume that differentiation can be exchanged with endpoint marginalization and that the corresponding continuity equations are well defined. Then $u_C$ satisfies
\begin{equation}
\partial_t p_t(x_t\mid c)+\nabla_{x_t}\!\cdot\!\left(p_t(x_t\mid c)u_C(x_t,t,c)\right)=0.
\label{eq:continuity-c}
\end{equation}
Moreover, for any square-integrable predictor $v(X_t,t,C)$,
\begin{equation}
\E\norm{v-U_t}^2=\E\norm{v-u_C}^2+\E\norm{U_t-u_C}^2.
\label{eq:cfm-fm}
\end{equation}
The second term is independent of $v$, so the sample CFM target and the inaccessible marginal-field target have the same population minimizer.
\end{proposition}

The proof is the usual conditional path marginalization and conditional-expectation projection, with the external context $C$ held fixed~\cite{lipman2022flow,lipman2024flow}. It requires no repeated occurrence of an exactly identical continuous observation $c$ in the dataset.

\subsection{The encoder bottleneck exposes representation deficiency}

Drift Variation autoencoder does not give $C$ directly to the flow decoder. The learned predictor has the restricted form
\begin{equation}
v_{\theta,\phi}(X_t,t,C)=v_\phi(X_t,t,E_\theta(C))=v_\phi(X_t,t,Z).
\label{eq:bottleneck}
\end{equation}
For a fixed encoder, define the best field available through its representation,
\begin{equation}
u_Z(X_t,t,Z)=\E[U_t\mid X_t,t,Z]=\E[u_C(X_t,t,C)\mid X_t,t,Z].
\label{eq:u-z}
\end{equation}

\begin{theorem}[Velocity CFM risk decomposition]
\label{thm:velocity-decomp}
For every encoder and field predictor for which the displayed second moments are finite,
\begin{align}
\cL_{\rm CFM}(E,v)
&:=\E\norm{v(X_t,t,Z)-U_t}^2\nonumber\\
&=\underbrace{\E\norm{U_t-u_C}^2}_{\cL_{\rm path}}
+\underbrace{\E\norm{u_C-u_Z}^2}_{\Delta_{\rm rep}^{v}(E)}
+\underbrace{\E\norm{u_Z-v(X_t,t,Z)}^2}_{\Delta_{\rm model}^{v}(E,v)}.
\label{eq:velocity-decomp}
\end{align}
The path term depends only on the data, augmentation, source, and schedules. With an unrestricted decoder, $v^*=u_Z$ and $\inf_v\cL_{\rm CFM}(E,v)=\cL_{\rm path}+\Delta_{\rm rep}^{v}(E)$.
\end{theorem}

The middle term is the information that was useful for the correct conditional velocity but lost when $C$ was replaced by $Z$. The final term is distinct: a finite decoder can fail even when the representation is sufficient, or it can prefer an easier but insufficient representation during joint optimization. Theorem~\ref{thm:velocity-decomp} is therefore a population accounting identity, not a promise about finite-network optimization.

\subsection{Clean prediction and posterior sufficiency}

In practice, the decoder is parameterized as a clean predictor
\begin{equation}
\widehat X=D_\phi(X_t,t,Z).
\label{eq:xpred}
\end{equation}
For the affine path, it induces the velocity
\begin{equation}
v_\phi(X_t,t,Z)=\frac{D_\phi(X_t,t,Z)-(1-\sigma_{\min})X_t}{\sigma_t}.
\label{eq:xpred-to-v}
\end{equation}
Exact velocity CFM is therefore equivalent away from $\sigma_t=0$ to a weighted clean-prediction loss with $w_{\rm vel}(t)=\sigma_t^{-2}$. The formulation allows any measurable positive time weight and defines
\begin{equation}
\cR_w(E,D)=\E\!\left[w(t)\norm{D(X_t,t,Z)-X}^2\right].
\label{eq:xpred-risk}
\end{equation}
Let
\begin{equation}
m_C(x_t,t,c)=\E[X\mid X_t=x_t,t,C=c],\qquad m_Z(x_t,t,z)=\E[X\mid X_t=x_t,t,Z=z].
\label{eq:denoisers}
\end{equation}

\begin{theorem}[Clean-prediction risk decomposition]
\label{thm:xpred-decomp}
If the displayed second moments are finite, then
\begin{align}
\cR_w(E,D)
&=\underbrace{\E\!\left[w(t)\norm{X-m_C}^2\right]}_{\cR_{\rm path}}
+\underbrace{\E\!\left[w(t)\norm{m_C-m_Z}^2\right]}_{\Delta_{\rm rep}^{x}(E)}\nonumber\\
&\quad+\underbrace{\E\!\left[w(t)\norm{m_Z-D(X_t,t,Z)}^2\right]}_{\Delta_{\rm model}^{x}(E,D)}.
\label{eq:xpred-decomp}
\end{align}
For a fixed encoder, the pointwise Bayes predictor is $D_E^*=m_Z$ wherever the effective time weight is positive.
\end{theorem}

This result explains why Eq.~\eqref{eq:xpred-risk} is not ordinary masked autoencoding even though both can use squared error. A deterministic reconstructor $g(C)$ has the single Bayes target $\E[X\mid C]$. Drift Variation autoencoder learns the family $m_Z(x_t,t,z)$ over noisy states and interior times. For Gaussian corruption, that family identifies the complete conditional distribution.

\begin{theorem}[Gaussian flow representation sufficiency]
\label{thm:main-sufficiency}
Let $Z=E(C)$, let $X_0=s_0\varepsilon$ with $s_0>0$ and $\varepsilon\sim\mathcal N(0,I_d)$ independent of $(X,C)$, and let $X_t=\alpha_tX+\sigma_tX_0$. Let $t\sim\rho$ be independent of $(X,C,X_0)$ and define $\cT_{\rm int}=\{t:\alpha_t>0,\ \sigma_t>0,\ 0<w(t)<\infty\}$. Assume $\rho(\cT_{\rm int})>0$, regular conditional probabilities exist, the conditional second moments are finite, and the density and conditional-expectation operations in Appendix~\ref{app:sufficiency-proof} are valid. Then
\begin{equation}
\boxed{\Delta_{\rm rep}^{x}(E)=0\quad\Longleftrightarrow\quad P(X\mid Z)=P(X\mid C)\ \text{almost surely}\quad\Longleftrightarrow\quad I(X;C\mid Z)=0.}
\label{eq:main-sufficiency}
\end{equation}
\end{theorem}

The positive-measure condition on $\cT_{\rm int}$ is important. If time is sampled continuously, merely naming one interior point of zero sampling mass is insufficient. Uniform, log-SNR, Beta, clipped-velocity, and exact-velocity schemes share the same ideal zero set when their effective measure assigns positive mass to interior times and the risk is finite.

\paragraph{Proof idea.} If $\Delta_{\rm rep}^{x}(E)=0$, Fubini's theorem gives at least one interior $t_*$ for which $m_C(y,t_*,c)=m_Z(y,t_*,z)$ almost everywhere. With $Y=\alpha_{t_*}X+\sigma_{t_*}s_0\varepsilon$, Gaussian differentiation gives
\begin{equation}
\nabla_y\log p_{t_*}(y\mid c)=\frac{\alpha_{t_*}m_C(y,t_*,c)-y}{\sigma_{t_*}^2s_0^2},\qquad \nabla_y\log p_{t_*}(y\mid z)=\frac{\alpha_{t_*}m_Z(y,t_*,z)-y}{\sigma_{t_*}^2s_0^2}.
\label{eq:score-identity}
\end{equation}
The two positive smooth densities have the same score and hence are equal after normalization. Their characteristic functions satisfy
\begin{equation}
\varphi_{Y\mid c}(\omega)=\varphi_{X\mid c}(\alpha_{t_*}\omega)\exp\!\left(-\tfrac12\sigma_{t_*}^2s_0^2\norm{\omega}^2\right),
\label{eq:charfun}
\end{equation}
and analogously given $z$. The Gaussian factor is nowhere zero and $\alpha_{t_*}>0$, so equality of the noisy laws implies equality of $P(X\mid C=c)$ and $P(X\mid Z=z)$. The reverse implication follows because posterior sufficiency makes the two Bayes denoisers identical. Appendix~\ref{app:sufficiency-proof} gives the measure-level details.

\begin{corollary}[Zero-set equivalence with conditional KL]
\label{cor:zero-set}
Under Proposition~\ref{prop:kl-decomp} and Theorem~\ref{thm:main-sufficiency},
\begin{equation}
\left\{E:\inf_Q\cK(E,Q)=0\right\}=\left\{E:\Delta_{\rm rep}^{x}(E)=0\right\}.
\label{eq:zero-set}
\end{equation}
This does not imply $\cR_w=\cK$, equality of their numerical values, or monotone KL decrease during parameter optimization.
\end{corollary}

\begin{corollary}[Collapse and posterior equivalence]
\label{cor:collapse}
If $I(X;C)>0$, a constant encoder cannot satisfy Eq.~\eqref{eq:main-sufficiency}. More generally, if two observations $c_1$ and $c_2$ induce different posteriors, a sufficient deterministic encoder cannot merge them except on null sets. The coarsest equivalence permitted by the theory is
\begin{equation}
c_1\sim c_2\quad\Longleftrightarrow\quad P(X\mid C=c_1)=P(X\mid C=c_2).
\label{eq:posterior-equivalence}
\end{equation}
\end{corollary}

Corollary~\ref{cor:collapse} gives an ideal anti-collapse statement, not an optimization guarantee. It also replaces unconditional view invariance with augmentation-adaptive sufficiency: weakly and strongly corrupted views of the same sample may appropriately have different representations because they specify different posterior uncertainty.

\subsection{Endpoint correctness and the joint ideal optimum}

Fix an encoder and define the representation-conditioned probability path
\begin{equation}
p_t(x_t\mid z)=\int p_t(x_t\mid x)p(x\mid z)\,\dd x.
\label{eq:path-z}
\end{equation}
Its marginal field is exactly $u_Z(x_t,t,z)=\E[U_t\mid X_t=x_t,t,Z=z]$. For the affine path, the Bayes clean predictor and field are related by
\begin{equation}
u_Z(x_t,t,z)=\frac{m_Z(x_t,t,z)-(1-\sigma_{\min})x_t}{\sigma_t}.
\label{eq:u-z-xpred}
\end{equation}

\begin{theorem}[Conditional endpoint correctness]
\label{thm:endpoint}
Fix $E$. Suppose the learned field equals $u_Z$ almost everywhere, the conditional ODE $\dot x_t=u_Z(x_t,t,z)$ has an appropriate unique flow, and its initial law is $P_0=\mathcal N(0,s_0^2I)$. Then the ODE marginal at every regular time equals $p_t(\cdot\mid z)$. If $\sigma_{\min}=0$, its endpoint satisfies
\begin{equation}
Q_{\rm FM}(X\mid Z=z)=P(X\mid Z=z).
\label{eq:endpoint-z}
\end{equation}
If in addition $\Delta_{\rm rep}^{x}(E)=0$, then
\begin{equation}
\boxed{Q_{\rm FM}(X\mid E(C))=P(X\mid C).}
\label{eq:joint-optimum}
\end{equation}
\end{theorem}

The two equalities in Theorem~\ref{thm:endpoint} separate the responsibilities of the decoder and encoder. Exact field fitting gives the posterior available through $Z$; zero representation deficiency makes that posterior equal to the one specified by $C$. Generation and representation learning are synchronized because failure of either component appears as a distinct term in the same population risk.

\begin{figure}[t]
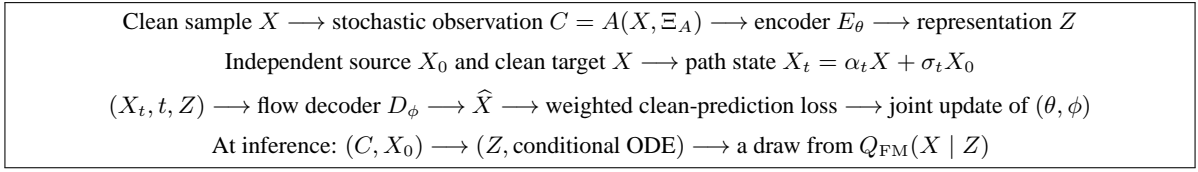

\centering
\fbox{\begin{minipage}{0.94\linewidth}\centering\small Clean sample $X$ $\longrightarrow$ stochastic observation $C=A(X,\Xi_A)$ $\longrightarrow$ encoder $E_\theta$ $\longrightarrow$ representation $Z$\\[5pt] Independent source $X_0$ and clean target $X$ $\longrightarrow$ path state $X_t=\alpha_tX+\sigma_tX_0$\\[5pt] $(X_t,t,Z)$ $\longrightarrow$ flow decoder $D_\phi$ $\longrightarrow$ $\widehat X$ $\longrightarrow$ weighted clean-prediction loss $\longrightarrow$ joint update of $(\theta,\phi)$\\[5pt] At inference: $(C,X_0)$ $\longrightarrow$ $(Z,\text{conditional ODE})$ $\longrightarrow$ a draw from $Q_{\rm FM}(X\mid Z)$\end{minipage}}
\caption{Drift Variation autoencoder uses one fixed sample-based objective to update the representation encoder and conditional flow decoder.
The clean target is used only during training; source noise supplies posterior diversity at inference.}
\label{fig:pipeline}
\end{figure}

\subsection{Multimodal posterior Flow Matching}

Let a paired multimodal sample be $\bX=(X^{(1)},\ldots,X^{(M)})\in\R^{d_1}\times\cdots\times\R^{d_M}$. The observation $\bC$ may contain stochastic corruptions of any subset of modalities together with a modality-presence pattern $S$. For example, it can contain visible image patches and audio while hiding depth and text. A multimodal encoder produces $\bZ=E(\bC)$, and the target is the joint posterior $P(\bX\mid\bC)$ rather than independent modality-wise marginals.

Use a full-rank Gaussian source $\beps\sim\mathcal N(0,I)$ independent of $(\bX,\bC)$ and a block path
\begin{equation}
\bX_t=\bA_t\bX+\bS_t\beps,
\label{eq:mm-path}
\end{equation}
where the simplest choice is $\bA_t=\alpha_tI$ and $\bS_t=\sigma_ts_0I$. Modality-specific schedules are allowed when $\bA_t$ is invertible and $\bS_t\bS_t^\top$ is positive definite on an interior set of positive sampling mass.

\begin{corollary}[Joint multimodal sufficiency]
\label{cor:multimodal}
Assume the conditions of Theorem~\ref{thm:main-sufficiency} on the product space, replace the scalar schedules by Eq.~\eqref{eq:mm-path}, and measure clean-prediction error with any positive-definite modality weight matrix. If the effective time measure gives positive mass to points where $\bA_t$ is invertible and $\bS_t\bS_t^\top$ is positive definite, then
\begin{equation}
\Delta_{\rm rep}^{\rm mm}(E)=0\quad\Longleftrightarrow\quad P(\bX\mid\bZ)=P(\bX\mid\bC)\quad\text{almost surely}.
\label{eq:mm-sufficiency}
\end{equation}
With an exact field and zero-noise endpoint, conditional ODE sampling returns the complete joint posterior. Changing which modalities are observed changes $\bC$ but not the objective, yielding reconstruction, cross-modal completion, any-subset generation, and unconditional generation as instances of one model.
\end{corollary}

The proof is the proof of Theorem~\ref{thm:main-sufficiency} in the concatenated space. The noisy characteristic function is multiplied by the nonzero factor $\exp(-\omega^\top\bS_t\bS_t^\top\omega/2)$, while invertibility of $\bA_t$ recovers every frequency of the joint clean law. Crucially, this identifies cross-modal dependencies and not only each marginal distribution.

For an any-subset implementation, the observation mask and Flow target must therefore have different roles. If $S$ is observed, only the modalities in $S$ enter the encoder condition, but the clean target remains the complete tuple $\bX$ for every $S$. All modality streams receive independent source noise, enter the joint noisy state $\bX_t$, contribute through a positive-definite loss, and evolve during ODE sampling. Visible modalities are generated and reconstructed; missing modalities are completed. The decoder receives the standard noisy Flow state and $\bZ$, never a separate copy of the raw clean observation. Training only the complement $\bX_{S^c}$ is a valid completion-only ablation, but it identifies $P(\bX_{S^c}\mid\bC)$ rather than the full joint posterior asserted by Corollary~\ref{cor:multimodal}.

The corollary is modality-agnostic at the probability level but not representation-free at the implementation level. Images, audio, video, depth, actions, and continuous text embeddings may use modality-specific tokenizers and output heads with a shared fusion backbone. Raw discrete text requires dequantization, a continuous latent codec, or a discrete Flow Matching construction with a separate identifiability proof; the Gaussian theorem should not be quoted unchanged for categorical tokens.

\subsection{Training and sampling}

\begin{table}[t]
\caption{Drift Variation autoencoder training with an affine Gaussian path.
Every prose step denotes a minibatch operation.}
\label{alg:training}
\centering
\small
\begin{tabular}{@{}rp{0.84\linewidth}@{}}
\toprule
\multicolumn{2}{@{}l}{\textbf{Input:} encoder $E_\theta$, clean predictor $D_\phi$, augmentation law $A$, time law $\rho$, weight $w$}\\
\midrule
1 & Sample a clean example $X$ and augmentation randomness, then form $C=A(X,\Xi_A)$.\\
2 & Encode only the observation: $Z=E_\theta(C)$.\\
3 & Sample $X_0\sim\mathcal N(0,s_0^2I)$ independently and $t\sim\rho$.\\
4 & Form $\sigma_t=1-(1-\sigma_{\min})t$ and $X_t=tX+\sigma_tX_0$.\\
5 & Predict $\widehat X=D_\phi(X_t,t,Z)$ and compute $w(t)\norm{\widehat X-X}^2$.\\
6 & Backpropagate jointly through $D_\phi$ and $E_\theta$; no stopped target or external teacher is required by the theory.\\
\midrule
\multicolumn{2}{@{}l}{\textbf{Output:} jointly trained representation encoder and posterior flow decoder}\\
\bottomrule
\end{tabular}
\end{table}

At sampling time, encode the available observation once, draw $X_0\sim\mathcal N(0,s_0^2I)$, and integrate
\begin{equation}
\frac{\dd X_t}{\dd t}=\frac{D_\phi(X_t,t,Z)-(1-\sigma_{\min})X_t}{\sigma_t}
\label{eq:sampling-ode}
\end{equation}
from $t=0$ to $t=1$. Different source draws produce different posterior samples. For $\sigma_{\min}=0$, network evaluations remain at $t<1$ because Eq.~\eqref{eq:sampling-ode} has a parameterization singularity at the endpoint; a numerical solver may integrate the final step to $1$ without querying the field exactly there.

\section{Controlled Validation}
\label{sec:evaluation}

The theory is a population statement. The evaluation tests its distinctive finite-model predictions on a distribution whose conditional factors and exact sampling oracle are known. The goal is not to simulate the semantic complexity of real image, text, and audio, but to distinguish representation use, marginal completion, and coherent joint posterior sampling without relying on a learned evaluator.

\subsection{CrossGeom-4 benchmark}
\label{sec:crossgeom}

CrossGeom-4 contains four independent continuous factors $u_0,u_{AB},u_{AC},u_{BC}$, each sampled from an equiprobable two-component Gaussian mixture with means $\pm1.5$ and standard deviation $0.25$. Three modalities expose overlapping factor subsets:
\begin{equation}
A:(u_0,u_{AB},u_{AC}),\quad
B:(u_0,u_{AB},u_{BC}),\quad
C:(u_0,u_{AC},u_{BC}).
\label{eq:crossgeom-factors}
\end{equation}
Observing one modality leaves exactly one factor absent from the condition but shared by the other two modalities. Correct joint generation must sample that factor from its conditional law and use the \emph{same draw} in both outputs. Observing two modalities determines every factor of the third.

The benchmark uses three coordinate levels. L1 applies fixed orthogonal transforms to Euclidean views; L2 adds an invertible elementwise nonlinearity; and L3 keeps $A$ in $\R^3$, embeds $B$ on $\mathbb S^3$ through inverse stereographic coordinates, and embeds $C$ on the upper hyperboloid $\mathbb H^3$. L3 is an extrinsic baseline: Flow Matching operates in ambient Euclidean coordinates and projects only the endpoint. No claim is made of intrinsic Riemannian transport.

\paragraph{Model and full-tuple objective.} A three-layer shared masked Transformer encoder with width 96 maps the visible modalities to one global token. A four-layer decoder receives time, the encoded condition, and noisy streams for all three modalities. The observation mask controls only encoder visibility. For every one of the eight masks, A, B, and C all receive independent Gaussian source noise, all contribute a dimension-normalized clean-prediction loss, and all are integrated during ODE sampling. Thus an A-only query generates $(A,B,C)$, including reconstruction of the visible A stream. The decoder receives no separate clean-condition bypass.

The joint decoder applies self-attention across all target tokens. The matched independent baseline uses one decoder per modality and therefore cannot couple residual source randomness across outputs through target-target attention. It retains the same encoder, hidden width, depth, full-tuple loss, and visibility schedule.

\paragraph{Protocol and metrics.} The protocol balances all eight visibility patterns and trains for 5,000 updates with batch size 256, learning rate $3\times10^{-4}$, clipped $\sigma_t^{-2}$ weighting with cap 100, and $\sigma_{\min}=0$. Evaluation uses a 32-step Euler solver, 128 held-out conditions, 16 samples per condition, and three seeds for every level/decoder pair, for 18 runs in total.

The evaluation linearly probes every factor from every nonempty visible subset. Factor-space MAE measures reconstruction of visible streams and prediction of factors known from the condition. With one visible modality, target-target MAE measures whether the two other outputs agree on their shared unknown factor. The evaluation compares the latter factor's marginal with the exact conditional oracle using one-dimensional Wasserstein distance (W1) and records whether repeated samples cover both mixture signs. A shuffled-condition test keeps the ODE source noise fixed. Unconditional evaluation reports all 16 sign modes, TV distance to the uniform mode law, and cross-modality factor disagreement. Appendix \ref{app:crossgeom} gives complete mean $\pm$ standard-deviation tables.

\subsection{Results}

\paragraph{The encoder is used and exposes observed factors.} Across all levels and both decoder types, mean observed-factor probe $R^2$ is $0.9990$-$0.9992$. For an A-only condition, the unavailable $u_{BC}$ factor has $R^2$ between $-0.023$ and $-0.016$, as expected under no leakage. With the same initial ODE noise, shuffling the joint model's valid conditions increases conditional error by $13.5\times$-$15.7\times$. The decoder therefore uses condition-specific information carried by the encoder rather than solving the task from the noisy Flow state alone.

\paragraph{All observed subsets generate the complete tuple.} The joint model's visible-stream factor MAE, averaged over all seven nonempty observation masks, is $0.0912$-$0.1041$. Its condition-to-target MAE is $0.1055$-$0.1583$ for single-input directions and $0.0923$-$0.1196$ for double-input directions. The all-visible ABC query also reconstructs all three generated streams, with MAE $0.0817$, $0.0818$, and $0.0944$ on L1-L3. These measurements directly distinguish the implementation from a complement-only objective in which visible target streams would be inactive.

\begin{table*}[t]
\centering
\small
\caption{Joint coupling under full-tuple Flow Matching.
Conditional MAE is similar for matched independent (Ind.) and joint decoders.
Joint target attention specifically reduces disagreement and distribution error for the unknown factor shared by two outputs.
Values are three-seed means; complete standard deviations appear in Appendix~\ref{app:crossgeom}.}
\label{tab:crossgeom-coupling}
\resizebox{\textwidth}{!}{%
\begin{tabular}{lccccccccc}
\toprule
& \multicolumn{2}{c}{Single-input conditional MAE} &
\multicolumn{3}{c}{Unknown target-target MAE} &
\multicolumn{3}{c}{Unknown-factor W1} & Uncond. joint \\
\cmidrule(lr){2-3}\cmidrule(lr){4-6}\cmidrule(lr){7-9}
Level & Ind. & Joint & Ind. & Joint & Reduction & Ind. & Joint & Reduction & reduction \\
\midrule
L1 & .1166 & .1176 & 1.6397 & .1186 & 92.8\% & .6776 & .2067 & 69.5\% & 88.8\% \\
L2 & .1145 & .1154 & 1.6278 & .1201 & 92.6\% & .6878 & .2669 & 61.2\% & 89.4\% \\
L3 & .1310 & .1338 & 1.5651 & .1542 & 90.1\% & .7143 & .1648 & 76.9\% & 86.6\% \\
\bottomrule
\end{tabular}}
\end{table*}

\paragraph{Marginal completion is not joint posterior sampling.} Table~\ref{tab:crossgeom-coupling} isolates the central claim. Independent and joint decoders have nearly identical known-factor accuracy, and both recover both signs of the residual mixture for essentially every context. Nevertheless, independent outputs assign incompatible values to the shared unknown factor. Joint attention reduces that disagreement by $90.1$-$92.8\%$ and reduces its W1 error by $61.2$-$76.9\%$. The effect therefore concerns posterior coupling, not an easier deterministic prediction or better mode coverage in either marginal alone.

\begin{figure*}[t]
\centering
\includegraphics[width=0.98\textwidth]{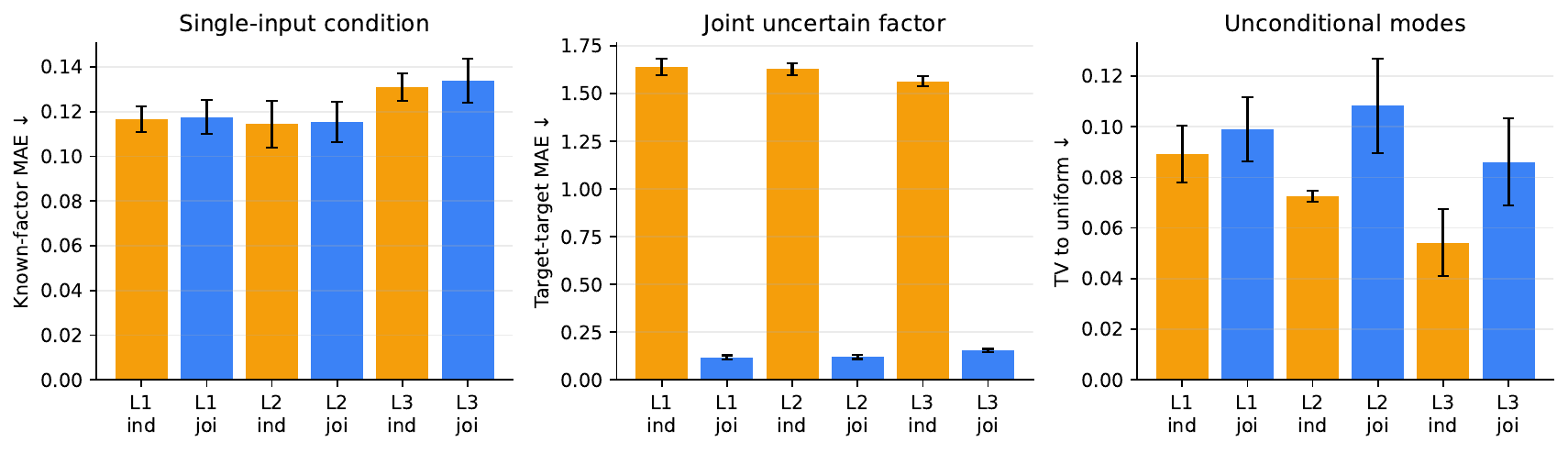}
\caption{Known-factor conditional MAE, disagreement on the jointly sampled unknown factor, and unconditional TV-to-uniform.
Joint and independent decoders have similar conditional accuracy, but only joint target attention coordinates residual uncertainty.
Error bars are standard deviations across the three seeds, computed with denominator $n$.}
\label{fig:crossgeom}
\end{figure*}

\paragraph{Unconditional generation remains the limiting case.} Both decoder types reach all 16 sign modes. Joint decoding reduces unconditional cross-modality factor disagreement by $86.6$-$89.4\%$, from $1.598$-$1.667$ to $0.177$-$0.213$. Its TV-to-uniform remains $0.086$-$0.108$, however, compared with a finite-sample reference of about $0.034$. The independent baseline has lower mode-frequency TV but fails joint factor consistency. Hence the strongest evidence is conditional: the same checkpoint exposes observed factors, uses the encoder condition, reconstructs visible targets, and couples residual uncertainty. L3 manifold-constraint residuals are below $3\times10^{-7}$ after endpoint projection, verifying valid outputs but not intrinsic manifold Flow Matching.

\section{Limitations}

The main equivalence is a population and realizability result. It does not show that stochastic gradient descent reaches a global optimum, that a finite encoder-decoder pair realizes the Bayes denoiser, or that a small empirical loss implies a quantitative bound on $I(X;C\mid Z)$. In finite capacity, representation deficiency and decoder approximation interact, and the easiest representation for a decoder need not be sufficient.

Posterior sufficiency is not minimal sufficiency or human semantic quality. A sufficient code may preserve instance identity, texture, or any detail that changes $P(X\mid C)$. Observation design and bottleneck structure determine which information is emphasized. CrossGeom-4 verifies factor accessibility and joint sampling in a deliberately low-dimensional system; it is not evidence of transfer to natural images, text, or audio. A real multimodal claim requires modality-specific tokenizers, substantially larger models, semantic and dense probes, and standard generative metrics.

The proof uses additive full-rank Gaussian corruption at a positively weighted interior time and injectivity of Gaussian convolution. Singular paths, non-Gaussian sources, discrete tokens, and intrinsic manifold-valued flows need separate identifiability arguments. CrossGeom-4 L3 uses ambient Euclidean Flow Matching followed by endpoint projection, not Riemannian Flow Matching. Raw categorical text likewise requires dequantization, a continuous codec, or a discrete-flow analysis.

Endpoint correctness assumes an exact field, a regular conditional ODE, and controlled integration. Finite-step solvers, the clean-prediction parameterization singularity at $t=1$, codecs, clipping, and projection add errors outside the theorem. Existing finite-error KL results require further smoothness and field regularity~\cite{su2025flow}; this work does not provide a general finite-risk conditional-KL or mutual-information bound.

The method also requires a condition-independent source and a genuine encoder bottleneck. If the decoder receives raw observations, crop metadata, a condition-dependent initialization, or another informative side channel, the sufficiency claim applies to the combined condition. In the full-tuple implementation, the decoder receives noisy Flow states for every modality but no separate clean visible stream; the visibility mask affects only the encoder. Finally, finite data can under-sample high-entropy modes. This appears in the results: all 16 unconditional modes are reached, but their frequency TV remains above the finite-sample reference even when joint consistency improves by nearly an order of magnitude.

\section{Conclusion}

Drift Variation autoencoder starts from one statistical object: the posterior of clean data given a stochastic observation. Conditional KL makes the desired decomposition explicit, while Conditional Flow Matching makes it sample based. Under an encoder bottleneck, the clean-prediction risk separates path variance, representation deficiency, and model approximation. For an affine Gaussian path, its representation term vanishes exactly when $P(X\mid Z)=P(X\mid C)$, and an exact conditional field transports common noise to that posterior. Generation and representation learning are synchronized in this precise zero-set and endpoint sense, not because Flow Matching numerically equals KL.

The product-space extension requires the complete multimodal tuple to remain the Flow target under every observation mask. CrossGeom-4 supports the resulting finite-model mechanism: the encoder linearly exposes observable factors, condition shuffling causes an order-of-magnitude degradation, visible modalities are reconstructed as active targets, and joint attention reduces unobserved shared-factor disagreement by $90$-$93\%$ relative to independent decoders without improving the easier deterministic conditional metric. Unconditional mode imbalance and the synthetic scale keep the claim narrow. The next empirical step is therefore not a stronger theorem but a matched natural-data evaluation of whether the same posterior-sufficient bottleneck also yields semantic and dense transfer.

\section*{AI Use Statement}

AI tools were used to assist with language polishing and the narrative construction of this manuscript. The author independently developed the research ideas, theoretical arguments, experiments, analyses, and conclusions, reviewed all AI-assisted edits, and takes full responsibility for the final content.

\bibliographystyle{unsrtnat}
{\small\bibliography{references}}

\clearpage
\appendix
\section{Notation, Assumptions, and Scope}
\label{app:assumptions}

All random variables are defined on a common probability space. Uppercase letters denote random variables and lowercase letters denote realizations. Equalities between conditional laws and conditional expectations are understood up to the corresponding null sets. Because $Z=E(C)$ is deterministic, $\sigma(Z)\subseteq\sigma(C)$ and $X\to C\to Z$ is a Markov chain. Regular conditional probabilities are used so that expressions such as $P(X\in B\mid C=c)$ and $P(X\in B\mid Z=z)$ can be chosen as measurable kernels.

The main text uses the following operation-level assumptions rather than a maximally sharp analytic package.
\begin{enumerate}[leftmargin=2em,itemsep=3pt]
\item \textbf{Conditional laws.} The data and representation spaces are standard Borel spaces, so regular conditional laws exist; in Theorem~\ref{thm:main-sufficiency}, the target $X$ takes values in $\R^d$ and has the conditional moments needed by the displayed squared losses.
\item \textbf{Independent Gaussian source.} $X_0=s_0\varepsilon$ with $s_0>0$, $\varepsilon\sim\mathcal N(0,I_d)$, and $\varepsilon\indep(X,C)$. This guarantees a common source for all conditions and a strictly positive smooth interior density.
\item \textbf{Schedules.} $\alpha_t$ and $\sigma_t$ are differentiable on the time interval used to define the velocity. For exact posterior endpoints, $\alpha_0=0$, $\sigma_0=1$, $\alpha_1=1$, and $\sigma_1=0$.
\item \textbf{Interior support.} The time variable $t\sim\rho$ is sampled independently of $(X,C,X_0)$. The set on which $\alpha_t>0$, $\sigma_t>0$, and the effective loss weight is positive and finite has strictly positive $\rho$ measure. This condition is stronger and more appropriate for a continuous time law than requiring one named point that may have zero mass.
\item \textbf{Interchange and score regularity.} Gaussian convolution may be differentiated under the conditional integral, Fubini and tower-property manipulations are valid, and equality of the resulting smooth scores identifies normalized densities. Standard finite-moment and domination assumptions are sufficient for the Gaussian formulas used here.
\item \textbf{Flow regularity.} Whenever equality of a vector field is used to infer equality of ODE marginals, the continuity equation and conditional ODE possess the required existence and uniqueness, and boundary or integrability conditions justify the manipulations.
\item \textbf{Realizability when stated.} Claims involving exact Bayes predictors or exact endpoint distributions assume the decoder family can realize the corresponding conditional expectation or velocity. The algebraic risk decompositions themselves do not require realizability.
\end{enumerate}

The conditional KL in Eq.~\eqref{eq:ideal-kl} may be interpreted with densities relative to a common dominating measure or directly as a KL between probability kernels. Proposition~\ref{prop:kl-decomp} applies whenever the chain rule is valid, with extended-real values handled in the usual way. The equivalence with conditional mutual information uses
\begin{equation}
I(X;C\mid Z)=\E_C\KL\!\left(P(X\mid C)\,\|\,P(X\mid Z)\right),
\label{eq:cmi-kernel}
\end{equation}
which follows because $P(X\mid C,Z)=P(X\mid C)$ for deterministic $Z=E(C)$.

The sample CFM target $U_t$ is fixed by $(X,X_0,t)$ and does not depend on the current encoder. The representation-conditioned Bayes field $u_Z$ does depend on the encoder through the sigma-field generated by $Z$. This distinction permits exact projection decompositions while preventing an interpretation of encoder training as direct KL gradient descent.

If the decoder receives an additional deterministic side channel $B=B(C)$, every theorem remains valid after replacing $Z$ by the combined representation $R=(Z,B)$. A claim about $Z$ alone then requires the stronger condition that $B$ does not carry additional posterior information or that the bypass is removed.

\section{Proof of the Conditional KL Decomposition}
\label{app:conditional-kl}

\begin{proof}[Proof of Proposition~\ref{prop:kl-decomp}]
Because $Z=E(C)$, insert the true representation-conditioned posterior between the numerator and denominator of Eq.~\eqref{eq:ideal-kl}:
\begin{align}
\log\frac{\dd P(X\mid C)}{\dd Q(X\mid Z)}
&=\log\frac{\dd P(X\mid C)}{\dd P(X\mid Z)}+\log\frac{\dd P(X\mid Z)}{\dd Q(X\mid Z)}.
\label{eq:rn-split}
\end{align}
Taking expectation under the joint law of $(X,C,Z)$ gives
\begin{align}
\cK(E,Q)
&=\E_C\KL\!\left(P(X\mid C)\,\|\,P(X\mid Z)\right)+\E_{X,Z}\log\frac{\dd P(X\mid Z)}{\dd Q(X\mid Z)}\nonumber\\
&=I(X;C\mid Z)+\E_Z\KL\!\left(P(X\mid Z)\,\|\,Q(X\mid Z)\right).
\end{align}
The first equality uses Eq.~\eqref{eq:cmi-kernel}; the second term is obtained by conditioning on $Z$. Nonnegativity shows that an unrestricted generator is minimized by $Q(\cdot\mid Z)=P(\cdot\mid Z)$, giving Eq.~\eqref{eq:kl-profile}. Conditional mutual information is zero exactly when $X\indep C\mid Z$, which is equivalent to Eq.~\eqref{eq:sufficiency} because $Z$ is a function of $C$.
\end{proof}

Proposition~\ref{prop:kl-decomp} also clarifies why minimality is absent. If $E$ is sufficient, augmenting $Z$ with any deterministic statistic of $C$ keeps the first term zero and does not worsen the unrestricted generator optimum. Minimality would require a separate ordering, dimension, entropy, rate, or bottleneck criterion.

\section{Conditional Path and Risk Decompositions}
\label{app:risk-decompositions}

\subsection{Marginalizing endpoint bridges}

For a fixed endpoint $X=x$, let $p_t(\cdot\mid x)$ satisfy the continuity equation with a bridge velocity $u_t(\cdot\mid x)$:
\begin{equation}
\partial_t p_t(x_t\mid x)+\nabla_{x_t}\cdot\left(p_t(x_t\mid x)u_t(x_t\mid x)\right)=0.
\label{eq:bridge-continuity-app}
\end{equation}
Conditioning on $C=c$ and integrating over $P(\dd x\mid c)$ yields
\begin{align}
\partial_t p_t(x_t\mid c)
&=-\nabla_{x_t}\cdot\int p_t(x_t\mid x)u_t(x_t\mid x)p(\dd x\mid c)\nonumber\\
&=-\nabla_{x_t}\cdot\left(p_t(x_t\mid c)\int u_t(x_t\mid x)\frac{p_t(x_t\mid x)p(\dd x\mid c)}{p_t(x_t\mid c)}\right).
\end{align}
The integral in parentheses is $\E[U_t\mid X_t=x_t,t,C=c]=u_C(x_t,t,c)$, proving Eq.~\eqref{eq:continuity-c}. For any square-integrable $v=v(X_t,t,C)$,
\begin{align}
\E\norm{v-U_t}^2
&=\E\norm{v-u_C}^2+\E\norm{u_C-U_t}^2+2\E\langle v-u_C,u_C-U_t\rangle.
\end{align}
The cross term is zero because $\E[U_t-u_C\mid X_t,t,C]=0$, proving Eq.~\eqref{eq:cfm-fm}.

\subsection{Nested projection for the velocity field}

Because $Z=E(C)$, the sigma-fields are nested:
\begin{equation}
\sigma(X_t,t,Z)\subseteq\sigma(X_t,t,C).
\end{equation}
The tower property gives $u_Z=\E[u_C\mid X_t,t,Z]$. Decompose
\begin{equation}
U_t-v=(U_t-u_C)+(u_C-u_Z)+(u_Z-v).
\end{equation}
The first difference is orthogonal in $L_2$ to every function of $(X_t,t,C)$. The second is orthogonal to every function of $(X_t,t,Z)$. Consequently all three pairwise cross terms vanish, which proves Theorem~\ref{thm:velocity-decomp}.

The representation term can also be written as an expected conditional variance:
\begin{equation}
\Delta_{\rm rep}^{v}(E)=\E\operatorname{Var}\!\left(u_C(X_t,t,C)\mid X_t,t,Z\right),
\label{eq:velocity-var}
\end{equation}
where the variance denotes the trace of the conditional covariance. This form emphasizes that the loss is incurred when distinct condition-level fields remain possible after conditioning on $Z$.

\subsection{Nested projection for clean prediction}

For a time-only weight $w(t)\ge0$, define the weighted inner product $\langle f,g\rangle_w=\E[w(t)f^\top g]$. Conditional expectation remains an orthogonal projection on all times where the weight is positive because the weight is measurable with respect to $t$. The tower property gives
\begin{equation}
m_Z=\E[m_C\mid X_t,t,Z].
\end{equation}
Expanding
\begin{equation}
X-D=(X-m_C)+(m_C-m_Z)+(m_Z-D)
\end{equation}
under the weighted inner product eliminates the cross terms and proves Theorem~\ref{thm:xpred-decomp}. In particular,
\begin{equation}
\Delta_{\rm rep}^{x}(E)=\E\!\left[w(t)\operatorname{Var}\!\left(m_C(X_t,t,C)\mid X_t,t,Z\right)\right].
\label{eq:xpred-var}
\end{equation}

\subsection{Clean prediction and affine velocity}

For $a=1-\sigma_{\min}$, the affine path satisfies $X_t=\sigma_tX_0+tX$ and $U_t=X-aX_0$. Eliminating $X_0=(X_t-tX)/\sigma_t$ gives
\begin{equation}
U_t=\frac{X-aX_t}{\sigma_t}.
\end{equation}
If $D$ induces $v=(D-aX_t)/\sigma_t$, then
\begin{equation}
\norm{v-U_t}^2=\frac{1}{\sigma_t^2}\norm{D-X}^2.
\label{eq:velocity-xpred-equivalence}
\end{equation}
Taking the conditional expectation of $U_t$ given $(X_t,t,Z)$ also yields Eq.~\eqref{eq:u-z-xpred}. Thus an exact clean-prediction Bayes rule produces the exact representation-conditioned marginal velocity wherever $\sigma_t>0$.

\section{Proof of Gaussian Flow Representation Sufficiency}
\label{app:sufficiency-proof}

Theorem~\ref{thm:main-sufficiency} is proved through three lemmas. Fix an interior time and write
\begin{equation}
Y=\alpha X+\beta\varepsilon,\qquad \alpha>0,\qquad \beta=\sigma s_0>0,\qquad \varepsilon\sim\mathcal N(0,I_d).
\label{eq:fixed-corruption}
\end{equation}

\begin{lemma}[Conditional Gaussian denoiser identity]
\label{lem:tweedie}
Under the interchange assumptions in Appendix~\ref{app:assumptions}, for almost every condition $c$,
\begin{equation}
\nabla_y\log p_Y(y\mid c)=\frac{\alpha\E[X\mid Y=y,C=c]-y}{\beta^2}.
\label{eq:tweedie-c}
\end{equation}
The analogous identity holds after conditioning on $Z=z$.
\end{lemma}

\begin{proof}
The conditional noisy density is the Gaussian convolution
\begin{equation}
p_Y(y\mid c)=\int\varphi_\beta(y-\alpha x)P(\dd x\mid c),
\end{equation}
where $\varphi_\beta$ is the density of $\mathcal N(0,\beta^2I_d)$. Differentiating under the integral gives
\begin{align}
\nabla_y p_Y(y\mid c)
&=\int-\frac{y-\alpha x}{\beta^2}\varphi_\beta(y-\alpha x)P(\dd x\mid c)\nonumber\\
&=p_Y(y\mid c)\frac{\alpha\E[X\mid Y=y,C=c]-y}{\beta^2}.
\end{align}
Gaussian convolution makes $p_Y(y\mid c)$ strictly positive, so division proves the claim. This is the conditional form of the denoiser-score identity underlying denoising score matching.
\end{proof}

\begin{lemma}[A Gaussian denoiser identifies the noisy law]
\label{lem:score-identifies}
Let $p$ and $q$ be strictly positive smooth normalized densities on connected $\R^d$. If their Gaussian posterior-mean functions in Lemma~\ref{lem:tweedie} agree almost everywhere, then $p=q$ almost everywhere.
\end{lemma}

\begin{proof}
Lemma~\ref{lem:tweedie} turns equality of posterior means into $\nabla\log p=\nabla\log q$ almost everywhere. Hence the weak gradient of $\log p-\log q$ is zero and this difference is constant on connected $\R^d$. Therefore $p=e^Kq$ for a constant $K$. Normalization forces $e^K=1$.
\end{proof}

\begin{lemma}[Gaussian smoothing is injective]
\label{lem:deconvolution}
If the conditional laws of $Y=\alpha X+\beta\varepsilon$ agree under two conditions and $\alpha,\beta>0$, then the corresponding conditional laws of $X$ agree.
\end{lemma}

\begin{proof}
For a conditional law $\mu$ of $X$, the characteristic function of $Y$ is
\begin{equation}
\varphi_Y(\omega)=\varphi_\mu(\alpha\omega)\exp\!\left(-\tfrac12\beta^2\norm{\omega}^2\right).
\end{equation}
The Gaussian multiplier never vanishes. Equality of the $Y$ laws therefore gives equality of $\varphi_\mu(\alpha\omega)$ for every $\omega$. Since $\alpha>0$, these arguments cover $\R^d$, and uniqueness of characteristic functions identifies the two clean laws.
\end{proof}

\begin{proof}[Proof of Theorem~\ref{thm:main-sufficiency}]
Suppose first that $\Delta_{\rm rep}^{x}(E)=0$. The integrand in Eq.~\eqref{eq:xpred-decomp} is nonnegative. Because the effective interior set $\cT_{\rm int}$ has positive $\rho$ measure, Fubini's theorem provides an interior time $t_*$ such that
\begin{equation}
m_C(X_{t_*},t_*,C)=m_Z(X_{t_*},t_*,Z)
\label{eq:denoisers-equal-rv}
\end{equation}
almost surely. For almost every $c$ and $z=E(c)$, the Gaussian conditional density of $X_{t_*}$ is positive on all of $\R^d$, so Eq.~\eqref{eq:denoisers-equal-rv} gives equality of the two denoiser functions for Lebesgue-almost every state. Lemmas~\ref{lem:tweedie} and~\ref{lem:score-identifies} imply
\begin{equation}
P(X_{t_*}\mid C=c)=P(X_{t_*}\mid Z=z).
\end{equation}
Lemma~\ref{lem:deconvolution} then gives $P(X\mid C=c)=P(X\mid Z=z)$ for almost every condition, proving posterior sufficiency.

Conversely, suppose $P(X\mid C)=P(X\mid Z)$ almost surely. This is equivalent to $X\indep C\mid Z$. Because $X_t$ is obtained from $X$ and independent source noise, the conditional joint kernel factors as
\begin{equation}
P(\dd c,\dd x,\dd x_t\mid z,t)=P(\dd c\mid z)P(\dd x\mid z)P(\dd x_t\mid x,t).
\end{equation}
It follows that $P(X\mid X_t,t,C,Z)=P(X\mid X_t,t,Z)$ and therefore $m_C=m_Z$ almost surely. Substitution into the definition of $\Delta_{\rm rep}^{x}$ makes it zero.

Finally, the equivalence between $P(X\mid C)=P(X\mid Z)$ and $I(X;C\mid Z)=0$ follows from Eq.~\eqref{eq:cmi-kernel} and nonnegativity of KL.
\end{proof}

\paragraph{Why $t=0$ alone is insufficient.} If $\alpha_0=0$, then $Y=\beta\varepsilon$ contains no scaled copy of $X$. The denoiser reduces to $\E[X\mid C]$ or $\E[X\mid Z]$, and the characteristic-function step cannot cover frequencies of $X$. Two different multimodal distributions can have the same mean, so equality at $t=0$ does not imply posterior sufficiency.

\section{Endpoint Correctness and the Multimodal Extension}
\label{app:endpoint-multimodal}

\subsection{Proof of conditional endpoint correctness}

Fix $E$ and $z$. Endpoint marginalization gives the path in Eq.~\eqref{eq:path-z}. Repeating the calculation in Appendix~\ref{app:risk-decompositions} with $C$ replaced by $Z$ shows that its conditional velocity is
\begin{equation}
u_Z(x_t,t,z)=\E[U_t\mid X_t=x_t,t,Z=z]
\end{equation}
and that $(p_t(\cdot\mid z),u_Z(\cdot,t,z))$ satisfies the conditional continuity equation. By assumption, the learned field equals $u_Z$ almost everywhere and the model ODE starts from the same common source. Uniqueness of the relevant continuity-equation or ODE flow therefore gives
\begin{equation}
q_t(\cdot\mid z)=p_t(\cdot\mid z)
\end{equation}
at every regular time. With $\alpha_1=1$ and $\sigma_1=0$, $p_1(\cdot\mid z)=P(X\mid Z=z)$, proving Eq.~\eqref{eq:endpoint-z}. Theorem~\ref{thm:main-sufficiency} supplies $P(X\mid Z)=P(X\mid C)$ when the representation deficiency is zero, proving Eq.~\eqref{eq:joint-optimum}.

If $\sigma_{\min}>0$, the same argument remains valid but identifies the smoothed endpoint
\begin{equation}
q_1(\cdot\mid z)=P(X\mid Z=z)*\mathcal N(0,\sigma_{\min}^2s_0^2I).
\end{equation}
The distinction is structural rather than a numerical detail: exact clean-posterior correctness requires a zero-noise endpoint or a separate deconvolution mechanism.

\subsection{Proof of joint multimodal sufficiency}

Let $\bX\in\R^D$ be the concatenation of all continuous modality latents, let $\bZ=E(\bC)$, and fix an interior time with invertible $\bA=\bA_t$ and positive-definite covariance $\boldsymbol\Sigma=\bS_t\bS_t^\top$. Write
\begin{equation}
\bY=\bA\bX+\boldsymbol\eta,\qquad \boldsymbol\eta\sim\mathcal N(0,\boldsymbol\Sigma),\qquad \boldsymbol\eta\indep(\bX,\bC).
\end{equation}
For a positive-definite loss matrix $W$, a zero multimodal representation term implies equality of conditional posterior means because $v^\top Wv=0$ only when $v=0$. The matrix Gaussian denoiser identity is
\begin{equation}
\nabla_y\log p_{\bY}(y\mid\bc)=\boldsymbol\Sigma^{-1}\left(\bA\E[\bX\mid\bY=y,\bC=\bc]-y\right),
\label{eq:mm-score}
\end{equation}
and similarly given $\bZ=\bz$. Equality of posterior means therefore yields equality of noisy scores and normalized noisy densities. The conditional characteristic function is
\begin{equation}
\varphi_{\bY\mid\bc}(\omega)=\varphi_{\bX\mid\bc}(\bA^\top\omega)\exp\!\left(-\tfrac12\omega^\top\boldsymbol\Sigma\omega\right).
\label{eq:mm-char}
\end{equation}
The Gaussian factor is nonzero and $\bA^\top$ is surjective, so equality of noisy laws identifies the entire joint characteristic function of $\bX$. This proves the forward direction of Corollary~\ref{cor:multimodal}; the reverse direction follows from the same conditional-independence factorization as in Theorem~\ref{thm:main-sufficiency}.

This argument would fail if separate losses identified only the marginal denoisers $\E[X^{(m)}\mid X_t^{(m)},Z]$ without conditioning on the joint noisy state. The Drift Variation autoencoder multimodal decoder receives the joint state $\bX_t$ and predicts the joint clean tuple, so the recovered score belongs to the smoothed joint density and preserves cross-modal dependence.

\subsection{Any-subset and unconditional branches}

Let $S\subseteq\{1,\ldots,M\}$ be the set of observed modalities and include $S$ in $\bC$. The same model then represents the family
\begin{equation}
P\!\left(X^{(1)},\ldots,X^{(M)}\mid \{A_m(X^{(m)}):m\in S\},S\right)
\label{eq:any-subset}
\end{equation}
over masks sampled during training. If $S=\varnothing$ is chosen independently of the data and produces a constant observation, Eq.~\eqref{eq:any-subset} reduces to the unconditional joint data law. If one modality is fully observed and another is hidden, it becomes cross-modal conditional generation. If every modality is partially observed, it becomes joint multimodal completion.

The observation mask is legitimate side information because it specifies which conditional distribution is intended. In contrast, passing hidden clean tokens or a learned encoding of them directly to the decoder would change the conditioning sigma-field and invalidate a claim that $\bZ$ alone is sufficient.

\section{Practical Implications and Broader Evaluation}
\label{app:practical}
\label{app:evaluation-details}

\subsection{Time weighting and representation pressure}

For the affine path, exact velocity regression corresponds to $w_{\rm vel}(t)=\sigma_t^{-2}$. With $\sigma_{\min}=0$, this weight diverges at the endpoint. Clipping it as
\begin{equation}
w_{\rm clip}(t)=\min\!\left(\sigma_t^{-2},w_{\max}\right)
\end{equation}
changes the numerical functional but not the pointwise Bayes clean predictor wherever the weight remains positive. Uniform, clipped, and exact weighting therefore share the ideal sufficient-encoder zero set under the theorem's support and integrability assumptions, while their finite-sample variance and optimization can differ substantially.

At $t=0$, the path state contains no clean signal and the denoiser compares only $\E[X\mid C]$ with $\E[X\mid Z]$. At an identifying interior time, the complete denoiser function identifies the posterior. Near a zero-noise endpoint, $X_t$ nearly reveals $X$, so the observed condition-use gap can shrink even for a useful representation. Time-resolved correct/shuffled/zero losses are thus more informative than a single time-averaged number.

\subsection{Source noise, augmentation entropy, and bypasses}

The common source $X_0\indep C$ both prevents condition leakage through the initial state and supplies residual sampling randomness. A condition-dependent source changes the conditioning sigma-field. Likewise, any direct crop geometry, raw clean token, calibration variable, or task identifier passed to the decoder becomes part of the representation certified by the theorem. The path state $X_t$ and time are intrinsic Flow Matching inputs; a separate clean copy of the observation is not.

Observation strength and flow time play different roles. Stronger masking or modality removal changes $P(X\mid C)$ and typically raises its entropy, whereas flow time changes the signal-to-noise ratio along a fixed posterior path. A full-mask branch is unconditional only when the mask event is independent of sample content. Multiple source draws reduce Monte Carlo noise and reveal posterior diversity, but they cannot recover modes missing from the dataset.

\subsection{Finite capacity and endpoint numerics}

With an unrestricted decoder, profiling the clean-prediction risk leaves only the fixed path term and representation deficiency. With finite networks, the model term depends on the encoder and joint training can prefer a lossy but easier code. Encoder token count, decoder cross-attention, width, and depth must therefore be ablated jointly. For $\sigma_{\min}=0$, training and integration should avoid evaluating the clean-prediction parameterization exactly at $t=1$. Clamping intermediate states changes the continuity equation; display clipping or manifold projection should occur only after integration unless the modified dynamics are analyzed explicitly.

\subsection{Natural-data evaluation protocol}

CrossGeom-4 tests the mechanism but not semantic transfer. A natural-image evaluation should pair frozen linear, $k$-nearest-neighbor, few-shot, and dense probes with FID, precision/recall, condition fidelity, and repeated-sample diversity from the \emph{same checkpoint}. High-entropy contexts require distributional metrics, best-of-$K$ fidelity, and consistency with visible evidence rather than single-output MSE alone. A conditional mean can score well under MSE while missing the posterior.

For multimodal systems, every observed subset should be evaluated against the complete generated tuple. Report per-modality fidelity, visible-stream reconstruction, cross-modal consistency, cross-modal retrieval, and dense probes. Repeated samples should change unobserved variables while preserving observed evidence. Marginal quality is insufficient: a compatibility metric or known factor evaluator must test whether simultaneously generated outputs share the same residual draw.

\begin{table}[t]
\centering
\small
\caption{Minimum information for evaluating a Drift Variation autoencoder instance.}
\label{tab:reporting}
\begin{tabular}{@{}L{0.24\linewidth}L{0.68\linewidth}@{}}
\toprule
Category & Required report \\
\midrule
Observation & Augmentations, strengths, mask law, content dependence, and all declared side information. \\
Representation & Encoder/token architecture, pooling, frozen and nonlinear probes, rank statistics, and checkpoint rule. \\
Flow & Source, schedules, time law, weighting/clipping, endpoint noise, solver, steps, and final-time handling. \\
Generation & Fidelity, precision/recall, condition consistency, repeated-sample diversity, and sample count. \\
Multimodal & Visible subsets, full generated targets, per-modality fidelity, visible reconstruction, and joint consistency. \\
Ablations & Shuffled/zero condition, source construction, time allocation, bottleneck size, and joint versus independent targets. \\
\bottomrule
\end{tabular}
\end{table}

No single finite-network metric validates Theorem~\ref{thm:main-sufficiency}. The empirical evidence should instead follow the decomposition: the encoder contains condition-specific predictive information, the decoder uses it, the source accounts for residual uncertainty, outputs preserve joint dependence, and representation and generation coexist at one declared checkpoint.

\section{CrossGeom-4 Details and Complete Results}
\label{app:crossgeom}

\subsection{Data construction and architecture}

For every example, the four factors follow
\begin{equation}
u_j\sim\tfrac12\mathcal N(-1.5,0.25^2)+
          \tfrac12\mathcal N( 1.5,0.25^2),
\qquad j\in\{0,AB,AC,BC\}.
\end{equation}
Each three-factor modality is mixed by a fixed seeded orthogonal matrix. L2 applies invertible $\operatorname{asinh}$ coordinates after mixing. L3 maps the $B$ chart to $\mathbb S^3$ by inverse stereographic projection and maps the $C$ chart to the upper hyperboloid by $y\mapsto(\sqrt{1+\lVert y\rVert^2},y)$. Exact inverse maps permit evaluation of latent factors without a learned perceptual model.

The encoder uses modality-specific scalar stems, modality and coordinate embeddings, one global token, three Transformer layers, width 96, and four attention heads. Unavailable modality tokens are padded. The decoder has four Transformer layers. The joint model self-attends across the concatenated A, B, and C target tokens; the independent baseline uses one decoder per modality.

All eight availability patterns occur equally often. Availability affects only the encoder. Under every pattern, including A-only and ABC, all three modalities receive independent Gaussian source states and are clean-prediction targets. The affine path uses $\sigma_{\min}=0$ and clipped $\sigma_t^{-2}$ weighting with cap 100. AdamW uses learning rate $3\times10^{-4}$, batch size 256, and 5,000 updates. Evaluation uses fixed-step Euler integration with 32 steps, 128 conditions, and 16 samples per condition. Every level/decoder setting uses seeds 0, 1, and 2. Tables report standard deviations across the three seeds, computed with denominator $n=3$.

\subsection{Complete aggregate tables}

\begin{table*}[t]
\centering
\scriptsize
\caption{Conditional factor MAE and unconditional metrics under full-tuple Flow Matching.
Every conditional query generates A, B, and C; direction labels identify the visible condition, not the only active targets.}
\label{tab:crossgeom-complete}
\resizebox{\textwidth}{!}{\begin{tabular}{llcccccccc}
\toprule
Level & Decoder & Uncond. TV$\downarrow$ & Uncond. joint MAE$\downarrow$ & A$\to$BC$\downarrow$ & B$\to$AC$\downarrow$ & C$\to$AB$\downarrow$ & AB$\to$C$\downarrow$ & AC$\to$B$\downarrow$ & BC$\to$A$\downarrow$ \\
\midrule
1 & independent & 0.0892 $\pm$ 0.0111 & 1.6462 $\pm$ 0.0460 & 0.1147 $\pm$ 0.0044 & 0.1177 $\pm$ 0.0072 & 0.1173 $\pm$ 0.0113 & 0.0962 $\pm$ 0.0029 & 0.0950 $\pm$ 0.0081 & 0.0943 $\pm$ 0.0006 \\
1 & joint & 0.0990 $\pm$ 0.0128 & 0.1847 $\pm$ 0.0096 & 0.1123 $\pm$ 0.0093 & 0.1212 $\pm$ 0.0076 & 0.1193 $\pm$ 0.0064 & 0.0990 $\pm$ 0.0124 & 0.1011 $\pm$ 0.0136 & 0.0966 $\pm$ 0.0096 \\
2 & independent & 0.0724 $\pm$ 0.0022 & 1.6674 $\pm$ 0.0426 & 0.1197 $\pm$ 0.0132 & 0.1132 $\pm$ 0.0130 & 0.1105 $\pm$ 0.0092 & 0.0975 $\pm$ 0.0153 & 0.0952 $\pm$ 0.0112 & 0.0831 $\pm$ 0.0116 \\
2 & joint & 0.1082 $\pm$ 0.0186 & 0.1771 $\pm$ 0.0116 & 0.1055 $\pm$ 0.0040 & 0.1204 $\pm$ 0.0118 & 0.1203 $\pm$ 0.0119 & 0.0923 $\pm$ 0.0026 & 0.0968 $\pm$ 0.0009 & 0.1061 $\pm$ 0.0202 \\
3 & independent & 0.0542 $\pm$ 0.0132 & 1.5977 $\pm$ 0.0220 & 0.1260 $\pm$ 0.0018 & 0.1543 $\pm$ 0.0112 & 0.1126 $\pm$ 0.0064 & 0.1150 $\pm$ 0.0083 & 0.0952 $\pm$ 0.0006 & 0.1072 $\pm$ 0.0090 \\
3 & joint & 0.0861 $\pm$ 0.0172 & 0.2133 $\pm$ 0.0164 & 0.1295 $\pm$ 0.0178 & 0.1583 $\pm$ 0.0143 & 0.1137 $\pm$ 0.0022 & 0.1196 $\pm$ 0.0156 & 0.1196 $\pm$ 0.0200 & 0.1127 $\pm$ 0.0079 \\
\bottomrule
\end{tabular}
}
\end{table*}

\begin{table*}[t]
\centering
\scriptsize
\caption{Representation, visible reconstruction, uncertainty, and condition-use evidence. ``Both modes'' is the fraction of contexts whose 16 samples contain both signs of the shared unknown factor. ``Shuffle'' is shuffled-condition MAE divided by matched-condition MAE with identical initial ODE noise.}
\label{tab:crossgeom-evidence}
\resizebox{\textwidth}{!}{\begin{tabular}{llccccccccc}
\toprule
Level & Decoder & Probe $R^2\uparrow$ & Random R$^2$ & Visible recon.$\downarrow$ & Single cond.$\downarrow$ & Double cond.$\downarrow$ & Unknown joint$\downarrow$ & Unknown W1$\downarrow$ & Both modes$\uparrow$ & Shuffle $\times\uparrow$ \\
\midrule
1 & independent & 0.9991 $\pm$ 0.0001 & 0.8227 $\pm$ 0.0858 & 0.0840 $\pm$ 0.0038 & 0.1166 $\pm$ 0.0059 & 0.0951 $\pm$ 0.0033 & 1.6397 $\pm$ 0.0432 & 0.6776 $\pm$ 0.0117 & 1.0000 $\pm$ 0.0000 & 15.6921 $\pm$ 0.6424 \\
1 & joint & 0.9992 $\pm$ 0.0000 & 0.8227 $\pm$ 0.0858 & 0.0918 $\pm$ 0.0066 & 0.1176 $\pm$ 0.0076 & 0.0989 $\pm$ 0.0114 & 0.1186 $\pm$ 0.0099 & 0.2067 $\pm$ 0.0539 & 1.0000 $\pm$ 0.0000 & 15.2957 $\pm$ 1.5546 \\
2 & independent & 0.9992 $\pm$ 0.0000 & -13.2173 $\pm$ 17.3786 & 0.0809 $\pm$ 0.0089 & 0.1145 $\pm$ 0.0104 & 0.0919 $\pm$ 0.0117 & 1.6278 $\pm$ 0.0313 & 0.6878 $\pm$ 0.0109 & 1.0000 $\pm$ 0.0000 & 16.4248 $\pm$ 1.9564 \\
2 & joint & 0.9992 $\pm$ 0.0000 & -13.2173 $\pm$ 17.3786 & 0.0912 $\pm$ 0.0069 & 0.1154 $\pm$ 0.0089 & 0.0984 $\pm$ 0.0078 & 0.1201 $\pm$ 0.0107 & 0.2669 $\pm$ 0.0309 & 1.0000 $\pm$ 0.0000 & 15.6814 $\pm$ 0.9940 \\
3 & independent & 0.9992 $\pm$ 0.0000 & 0.7396 $\pm$ 0.2158 & 0.0910 $\pm$ 0.0050 & 0.1310 $\pm$ 0.0060 & 0.1058 $\pm$ 0.0054 & 1.5651 $\pm$ 0.0261 & 0.7143 $\pm$ 0.0133 & 0.9991 $\pm$ 0.0012 & 14.1745 $\pm$ 0.6261 \\
3 & joint & 0.9990 $\pm$ 0.0001 & 0.7396 $\pm$ 0.2158 & 0.1041 $\pm$ 0.0069 & 0.1338 $\pm$ 0.0099 & 0.1173 $\pm$ 0.0122 & 0.1542 $\pm$ 0.0080 & 0.1648 $\pm$ 0.0516 & 1.0000 $\pm$ 0.0000 & 13.5004 $\pm$ 1.1064 \\
\bottomrule
\end{tabular}
}
\end{table*}

The negative mean random-feature $R^2$ on L2 is a valid out-of-sample result, not a clipped value. Random high-dimensional features make the fixed ridge probe poorly conditioned and coordinate dependent. Trained representations, in contrast, remain stable across levels, and factors unavailable under a single observed modality remain near chance.

\subsection{Oracle construction and metric interpretation}

The conditional oracle copies every factor determined by the visible modalities and independently resamples each undetermined factor from its true mixture. It then renders the complete A/B/C tuple. Consequently, condition-to-target error has a deterministic zero oracle, whereas an undetermined factor is evaluated distributionally rather than against one arbitrary ground-truth draw. The unconditional oracle samples all four factors and provides a finite-sample TV calibration of $0.0339\pm0.0003$ for 2,048 samples.

Both decoder structures attain essentially perfect two-sign coverage with 16 draws per context. This marginal metric cannot establish joint generation. The unknown-factor target-target MAE tests whether two output modalities use the same posterior draw, and it is the metric on which joint and independent decoders separate by roughly one order of magnitude. The W1 metric complements it by comparing the averaged recovered factor to the exact conditional marginal.

Shuffling preserves the empirical support of each observed modality and reuses the same initial source noise, making it the primary condition-use test. A zero condition is also recorded in the long-form results but is an out-of-support intervention and is not used for headline effect sizes. All headline values, the aggregate JSON, and the long-form CSV are included with the source package.

\end{document}